\pdfoutput=1
\documentclass[11pt]{article}

\usepackage[margin=1in]{geometry}
\usepackage[hyphens]{url}
\usepackage{graphicx}
\usepackage{natbib}
\usepackage{caption}
\usepackage{algorithm}
\usepackage{algorithmic}
\usepackage{booktabs}
\usepackage{amsmath,amsthm}
\usepackage{amssymb}
\usepackage{xcolor}
\usepackage{hyperref}
\newtheorem{theorem}{Theorem}

\newtheorem{definition}{Definition}
\newtheorem{lemma}{Lemma}
\newtheorem{proposition}{Proposition}

\newcommand{\E}{\mathbb{E}}
\newcommand{\Prob}{\mathbb{P}}
\newcommand{\R}{\mathbb{R}}
\newcommand{\abs}[1]{\left|#1\right|}
\newcommand{\norm}[1]{\left\|#1\right\|}
\newcommand{\ip}[2]{\langle #1, #2 \rangle}
\newcommand{\sgn}{\operatorname{sign}}

\title{Fast Quantized Johnson--Lindenstrauss Transforms}
\author{%
    Angelos Korakitis\thanks{National and Kapodistrian University of Athens, \texttt{akorakitis@di.uoa.gr}.}
\and%
    Vasileios Nakos\thanks{National and Kapodistrian University of Athens \texttt{vasilisnak@di.uoa.gr}.}}
\date{}

\begin{document}

\maketitle

\begin{abstract}
Quantization has emerged as one of the primary techniques for reducing the memory footprint of large language models during inference. A prominent example is Quantized Johnson–Lindenstrauss (QJL), which compresses KV caches using 1-bit asymmetric quantization while approximately preserving the inner products required by the attention mechanism. Recent theoretical and practical work aims at replacing the dense Gaussian projections used by the aforementioned transforms with more structured embeddings — such as compositions of randomized Hadamard transforms — while retaining the same embedding dimension.

In this paper, we introduce two  constructions based on a composition of randomized Hadamard transforms, albeit with a Gaussian rather than Rademacher diagonal in the second stage. The first adds Gaussian dithering to smooth the quantizer, while the second inserts a random permutation to regularize the covariance structure of the resulting Gaussian measurements.

We prove embedding-dimension and running-time guarantees analogous to those of Fast JL in the non-quantized setting. The first construction achieves \(m=O(\varepsilon^{-2}\log^2(d/\delta))\), while the second achieves \(m=O(\varepsilon^{-2}\log(d/\delta))\) under a restriction on \(m\). Both constructions can be applied in \(O(d\log d)\) time. 


\end{abstract}

\section{Introduction}

Vector quantization (VQ) in Euclidean space has emerged as a crucial technique in the fine-tuning and deployment of AI models, as well as in vector databases, which heavily operate on similary search. The core idea behind VQ is to keep low-bit representations of high dimensional vectors, often obliviously to the considered data, such that quantities of interest, including inner products or mean squared errors (MSEs), are retained with minimal error. 

 VQ becomes particular significant in transformer-based architectures, which face a significant challenge due to the self-attention mechanism. The memory footprint of the Key-Value (KV) cache can often be a bottleneck; though the expressiveness of high dimensional vectors drives much of the success of LLMs, storage issues may prevent further scalability. This has lead researchers to the following question: \emph{how can we compress keys, values and queries without sarcificing too much accuracy?} This would (and has) allow(ed) LLMs to process much longer context windows and support larger batch sizes without meeting the barrier of current GPU memory. 

Additionally, the VQ framework is of direct application to nearest neighbor search in vector databases~\citep{qdrant, Pg, Pinecone}, where searches are often performed with respect to inner product or cosine similarity measure~\citep{jegou2011product, gao2025practical, ge2013optimized}. Retrieval-augmented generation and information retrieval have become ubiquitous, as they allow AI models to access external data, up-to-date data. These techniques allow factual groundings, cost savings (as expensive re-training can be bypassed, since local files can be accessed dynamically) and semantic matching.

A significant building block behind VQ is the application of a uniformly random rotation to the input vectors and then quantizing down to a small number of bits. This approach has led to consider quantized analogues of the the Johnson-Lindenstrauss transforms, with asymmetric inner product estimators for KV cache compression~\citep{zandieh2025qjl}. The aforementioned embeddings, namely QJL and the $4$-bit analogue PolarQuant \citep{zandieh2025qjl, han2026polarquant}, are behind the recent success of Google's Turboquant \citep{zandieh2026turboquantA}. The general pipeline of applying a random linear transform followed by coordinatewise scalar quantization has a long history \citep{charikar2002similarity, dong2008asymmetric,vargaftik2021drive,vargaftik2022eden, chee2023quip, tseng2024qtipA, tseng2024quipB, ashkboos2024quarot, gao2024rabitq, gao2025practical, zandieh2026turboquant}.

These embeddings achieve a favorable balance between dimensionality reduction and the error introduced, but are typically dense and therefore slow to apply. Thus, both researchers and practicioners have suggested alternative constructions, such as composing series of Randomized Hadamard transforms (RHTs), which allow embedding in $O(d \log d)$ time \citep{agarwal2024hadacore, zandieh2026turboquant, vargaftik2022eden, vargaftik2021drive, feng2026provable, ben2026quantizing, tseng2024quipB, le2013fastfood, dasgupta2011fast, ailon2009fast, sarlos2006improved, krahmer2011new, ordentlich2025optimal, konevcny2016federated, dorfman2023docofl, caldas2018expanding, ben2025better, basat2024optimal, andoni2015practical}. Recent work \citep{ben2026quantizing} has shown that composing two such RHTs yields a transform $T$ such that $Tx$ is close in Kolmogorov distance to randomly rotating $x$.

In this work, we continue this line of research by designing Quantized Johnson-Lindenstrauss transforms which enable $O(d \log d)$ time and have analogous guarantees to the renowned Fast JL transforms in the unquantized case. Our constructions are careful twists of the aforementioned framework of composing two RHTs and provide results which, to the best of our knowledge, our new, and are incomparable to the work of \citep{ben2026quantizing}.


\paragraph{Our contributions.}

We introduce two variants of a fast asymmetric quantized JL estimator. The first, Dithered Fast Quantized Johnson--Lindenstrauss \textsc{D-FQJL}, uses dithered quantization, while the second, Permutation-based Fast Quantized Johnson--Lindenstrauss \textsc{P-FQJL}, combines RHTs with a random permutation.
The sketch used in our permutation-based variant is the structured transform previously considered in other settings~\citep{dasgupta2011fast,le2013fastfood}. Our objective and analysis, however, are quite different, as we argue that the covariance of gaussian ensembles arising throughout is close to the identity.

Both constructions retain the asymmetric structure of QJL: one vector is
stored as an $m$-bit binary sketch obtained by applying a 1-bit quantizer to
each of the $m$ projected coordinates, whereas the other is represented by an
$m$-dimensional floating-point sketch. If both
projected vectors were sign-quantized, their bit agreements would instead
give an unbiased estimator of their angle~\citep{charikar2002similarity}.
Recovering the inner product would then require applying the cosine function
to the estimated angle and multiplying by the vector norms; this nonlinear
plug-in step is generally biased. Quantizing only one vector avoids this issue
and yields the sign--linear identity needed for exact unbiasedness. It also
reduces the storage for a pair of vectors from $2m$ floating-point numbers
(FPNs) to $m+1$ FPNs and $m$ bits. The additional scalar stores the
normalization or quantization scale associated with the binary sketch.

Our structured representation further reduces the space required to store the sketching operator. Whereas a dense $m\times d$ Gaussian matrix requires $O(md)$ our construction requires only $O(m+d)$ space. Moreover, the algorithm is data-oblivious: the sketch is sampled independently of the keys, values, and queries, and therefore requires no data-dependent training, optimization, or preprocessing.

We prove inner-product estimation guarantees for both variants and show how they can be applied directly to KV-cache compression. A cached key can be stored as an $m$-bit binary sketch while the corresponding query is transformed to a compact representation. 


Our results also connect TurboQuant and EDEN/Drive randomized quantization
methods~\citep{zandieh2026turboquant,vargaftik2022eden, vargaftik2021drive,ben2026note}. Both lines
of work use or advocate RHTs for fast practical
implementations. However, directly replacing the Gaussian projection in the
asymmetric sign-linear estimator by $H D_\pm$ does not, in general, preserve
exact unbiasedness: the Gaussian sign identity underlying QJL no longer
applies, and the RHT substitution can be biased on worst-case inputs. Our
constructions address this gap by retaining fast-transform complexity while
recovering an exactly unbiased asymmetric estimator. Our embeddings are twists of compositions of two RHTs, similarly to \citep{ben2026quantizing}, see Related work for more details.

\subsection{Related Work}

\paragraph{KV-cache compression.}
Prior work reduces KV-cache costs through direct low-bit
quantization~\citep{liu2024kivi}, rotation-based outlier
suppression~\citep{ashkboos2024quarot,liu2025spinquant}, and randomized
quantization designed to preserve geometric quantities such as inner products
or distances~\citep{zandieh2025qjl,zandieh2026turboquant,gao2024rabitq}. Our
work follows the third direction, replacing dense random projections with fast
structured transforms. 
A complementary approach compresses the sequence
dimension by retaining only selected cached tokens~\citep{kochetkova2026streaming}.

\paragraph{Randomized Hadamard and fast orthogonal transforms.}
RHTs provide fast, hardware-efficient mixing in
$O(d\log d)$ time~\citep{agarwal2024hadacore} and are widely used to reduce
coherence before quantization~\citep{ashkboos2024quarot,liu2025spinquant,
chee2023quip}. Their structured output is more difficult to analyze than a
dense random rotation, and classical theory therefore focused mainly on JL
and subspace-embedding guarantees~\citep{ailon2009fast,ailon2013almost,
krahmer2011new,sarlos2006improved}. More recent work studies uniform
approximation and/or quantization guarantees for Gaussian-diagonal or randomized
Hadamard constructions~\citep{cherapanamjeri2022uniform,feng2026provable,
schuchman1964dither,ben2026quantizing, zilca2026approximating}. In~\citep{ben2026quantizing} in particular the authors prove applying that the composition of two RHTs yields a vector the marginal distribution of every fixed coordinate of which is close to Kolmogov distance to the normalized rotated: the gap is $O(d^{-1/2})$ and is incomparable to our guarantee. For a thorough practical discussion of this framework see also the blog of Christiani~\citep{christiani2025rotational}.



\section{Preliminaries}
\label{gen_inst}

\subsection{Notation}

For a positive integer $n$, we write $[n]:=\{1,\dots,n\}$. For a vector
$v\in\R^d$ and $p\ge1$, we define the $p$-th norm of $v$ by
$\norm{v}_p:=(\sum_{i=1}^d\abs{v_i}^p)^{1/p}$, while
$\norm{v}_\infty:=\max_{i\in[d]}\abs{v_i}$. We write
$\ip{u}{v}=u^\top v$ for the inner product.
For the sign function,
we use the convention $\sgn(x)=1$ if $x\ge0$ and $\sgn(x)=-1$ if $x<0$;
for vectors, $\sgn$ is applied entrywise.


A random variable $X$ has the \emph{Rademacher}$(p)$ distribution if
$\Prob(X=1)=p$ and $\Prob(X=-1)=1-p$. A random
variable is \emph{standard Gaussian} if it has distribution
$\mathcal{N}(0,1)$. For $g\sim\mathcal{N}(0,I_d)$, let
$D_g:=\operatorname{diag}(g)$. Similarly, let
$\zeta_i\overset{\mathrm{i.i.d.}}{\sim}\operatorname{Rademacher}(1/2)$ for
$i\in[d]$, and set $D_\pm:=\operatorname{diag}(\zeta)$.

Throughout the work, we assume $d=2^s$. Let
$H\in\{\pm1/\sqrt d\}^{d\times d}$ be the normalized Hadamard
matrix~\citep{horadam2012hadamard}. Thus, $H H^\top=I_d$, so $H$ is orthogonal.

Let $I\subseteq[d]$ be a uniformly random subset of size $m$, written as $I={i_1<\cdots<i_m}$. The coordinate-selection matrix $S_I\in\{0,1\}^{m\times d}$ is obtained by placing $e_r\in\mathbb{R}^m$ in column $i_r$ for each $r\in[m]$. A permutation matrix $\Pi\in\{0,1\}^{d\times d}$ has exactly one
nonzero entry in each row and column. Equivalently, for a permutation $\pi$
of $[d]$, we use the convention $\Pi_{ij}=1$ if $j=\pi(i)$ and $0$ otherwise,
so $(\Pi v)_i=v_{\pi(i)}$.
For $v\in\R^n$, define softmax componentwise by
$\operatorname{softmax}(v)_i:=\exp(v_i)/\sum_{j=1}^n\exp(v_j)$ for $i\in[n]$.



 


\subsection{Attention Mechanism}

We now review the attention mechanism used by autoregressive transformers. Given a sequence of tokens, the model produces three embeddings for each token position: a query, a key, and a value. We
denote these embeddings at position $i$ by $q_i,k_i,v_i\in\R^d$, respectively. These embeddings allow us to derive the contextual representation of each token.

To produce the contextual representation of a specific token, the model compares its query with all keys in the current context, producing a vector of similarity scores. These scores are normalized by the softmax function into a vector of \emph{attention weights}, or \emph{attention scores}. The contextual representation of the token is then the convex combination of the value vectors, weighted by these attention scores.

Formally, at every step $j$, in order to generate the $(j+1)$-st
token, define the attention-score vector
\begin{align}
\operatorname{score}\,(j)
&:=
\operatorname{softmax}
\left(
\frac{K_jq_j}{\sqrt{d}}
\right)
\in\R^j.
\label{eq:attention-scores}
\end{align}
The corresponding contextual representation is
\begin{align}
\operatorname{Attn}(q_j,K_j,V_j)
&:=
\operatorname{score}\,(j)^{\mathsf T}V_j,
\label{eq:streaming-attention}
\end{align}
where $K_j,V_j\in\R^{j\times d}$ are the matrices whose rows are
$k_1,\dots,k_j$ and $v_1,\dots,v_j$, respectively. Thus $K_jq_j\in\R^j$
is the vector of inner products between $q_j$ and the rows of $K_j$.

As is apparent from Equation \eqref{eq:streaming-attention}, evaluating attention at step $j$ requires access to every key and value embedding generated so far. These are stored in what is commonly known as the \emph{KV cache}, so that they need not be recomputed at each subsequent step. Its memory footprint grows linearly with $j$ and, kept in full precision, can become substantial for long contexts.
Likewise, computing the attention scores involves $j$ inner products of dimension $d$, giving a per-token complexity of $O(jd)$. It is precisely this cost of storing $K_j$ and computing $K_jq_j$ that motivates our work: we reduce the memory footprint of the KV cache using quantization and dimensionality reduction.

\section{Fast Quantized Johnson--Lindenstrauss Transforms}
\label{sec:fqjl}

We now develop fast structured analogues of the asymmetric quantized Johnson--Lindenstrauss estimator. The goal is to replace the dense Gaussian projection used in QJL by a transform that can be applied in nearly linear time, while preserving the same asymmetric storage format: the key is stored only through quantized sketch values and its Euclidean norm, whereas the query is sketched at evaluation time.

The section is organized as follows. We first describe the common
Gaussian--Hadamard sketch underlying our constructions and explain why it has
the correct Gaussian marginals. This yields an unbiased estimator, but not
concentration, because the rows of the fast transform are dependent. The two
estimators introduced below address this difficulty in complementary ways.
The dithered estimator \textsc{D-FQJL} smooths the sign quantizer and uses
Gaussian concentration, whereas the permutation-based estimator
\text{P-FQJL} introduces additional randomization that enables a matrix
Bernstein argument.

\subsection{A Shared Gaussian--Hadamard Sketch}
\label{sec:fqjl-shared}

The asymmetric QJL estimator~\citep{zandieh2025qjl} relies on the following
elementary identity for linear functionals of a Gaussian vector.

\begin{lemma}[Gaussian sign identity]
\label{lem:gaussian-sign-identity}
Let $g\sim\mathcal{N}(0,I_d)$, and let $a,b\in\R^d$ with $b\neq0$. Then
\begin{align}
\label{eq:gaussian-sign-identity}
\E\!\left[
  \sqrt{\frac{\pi}{2}}\,\norm{b}_2\,
  \ip{a}{g}\,\sgn\!\bigl(\ip{b}{g}\bigr)
\right]
&= \ip{a}{b}.
\end{align}
\end{lemma}
\noindent The proof is provided in the supplementary material.

With a dense Gaussian sketch, each measurement uses an independent Gaussian
row, so Lemma~\ref{lem:gaussian-sign-identity} yields both unbiasedness and
concentration through averaging. However, storing and applying an
$m\times d$ dense Gaussian matrix requires $O(md)$ space and time. 

Lemma \ref{lem:gaussian-sign-identity} identifies the structure needed for an asymmetric unbiased
inner-product estimator. We therefore seek a linear map whose row measurements
form such pairs of Gaussian linear functionals and whose $m$ outputs can be
computed together in $O(d\log d)$ time using the fast Walsh--Hadamard transform. We
begin with a simple prototype and then refine it.

\paragraph{First approach.}

Consider the sketch
\begin{align}
\label{eq:fqjl-T0}
T_0 &:= S_I\sqrt{d}\,H D_g.
\end{align}
For every $v\in\R^d$ and sampled index $i\in I$, the corresponding
measurement is
\begin{align}
\label{eq:fqjl-prototype-row-notation}
(T_0v)_i
&= \ip{(\sqrt{d}\,H_{ij}v_j)_{j=1}^d}{g},
\end{align}
where $\sqrt{d}\,H_{ij}\in\{\pm 1\}$. Thus multiplying the $j$-th coordinate
by $\sqrt{d}\,H_{ij}$ preserves norms and inner products. Applying
Lemma~\ref{lem:gaussian-sign-identity}, one can easily see that \ref{eq:fqjl-T0} provides an unbiased estimator of the inner product.

Moreover, both $T_0q$ and $T_0k$ can be computed in
$O(d\log d)$ time using the fast Walsh--Hadamard transform.

Although $T_0$ is unbiased and fast to apply, the row dependence described
above means that it need not concentrate: correct one-row marginals do not
guarantee that averaging over additional rows reduces the variance.

\begin{proposition}[Variance of $\hat e_0$ need not vanish]
\label{prop:t0-variance}
There exist $q,k\in\R^d$ for which
$\operatorname{Var}(\hat e_0(q,k))$ is bounded below by an absolute constant,
independently of $m$.
\end{proposition}
\begin{proof}
Take $q=k=e_1$. Then $\hat e_0(e_1,e_1)=\sqrt{\pi/2}\,\abs{g_1}$, so
$\operatorname{Var}(\hat e_0(e_1,e_1))=\pi/2-1$ regardless of $m$.
\end{proof}
\noindent The proof is provided in the supplementary material. 

Thus, for this input, increasing the number of sampled rows does not reduce
the estimation variance.

\paragraph{Second approach.}


The failure in Proposition~\ref{prop:t0-variance} arises from inputs whose
energy is concentrated on a small number of coordinates. To mitigate this
effect, we first apply an RHT $H D_\pm$. This
orthogonal preprocessing step preserves $\ell_2$ norms and inner products and, with
high probability, bounds the $\ell_\infty$ norm of the transformed vector in
terms of its $\ell_2$ norm.

\begin{lemma}[Hadamard spreading]
\label{lem:spreading}
Let $x\in\R^d$ be nonzero and let $\delta>0$. Then, with probability at least
$1-\delta/4$ over $D_\pm$,
\begin{align*}
\norm{H D_\pm x}_\infty
\le
\norm{x}_2\sqrt{\frac{2\log(16d/\delta)}{d}}.
\end{align*}
\end{lemma}
\noindent  The proof is provided in the supplementary material. 

Composing this spreading step with the prototype gives the shared Gaussian--Hadamard sketch
\begin{align}
\label{eq:fqjl-transform}
T &:= S_I\sqrt{d}\,H D_g H D_\pm.
\end{align}
For a nonzero key $k$, define the $m$-bit code
$\mathcal{Q}(k):=\sgn(Tk)\in\{-1,+1\}^m$ and store it together with
$\norm{k}_2$. Given a query $q$, estimate the inner product by
\begin{align}
\label{eq:fqjl-estimator}
\hat e(q,k)
&:= \frac{\sqrt{\pi/2}}{m}\,
\norm{k}_2\ip{Tq}{\mathcal{Q}(k)}.
\end{align}
Conditioning on the additional randomness in
$D_\pm$, the same calculation gives
\begin{align}
\label{eq:fqjl-unbiasedness}
\E_g\!\left[\hat e(q,k)\mid I,D_\pm\right] &= \ip{q}{k}.
\end{align}

The estimator therefore retains the asymmetry of QJL~\citep{zandieh2025qjl}:
the key is stored as an $m$-bit code together with its norm, whereas the query
is transformed into the $m$-dimensional vector $Tq$.
Thus, the query also undergoes dimensionality reduction, but its new representation is
not quantized. The estimator is unbiased and can
be evaluated in $O(d\log d)$ time. 

Establishing concentration is more delicate. The sampled rows of $T$ are
dependent because they reuse the same Gaussian vector $g$, and this dependence
is further complicated by the nonsmooth $\operatorname{sign}$ function. The next two
subsections develop complementary constructions to address these two
difficulties.

\subsection{D-FQJL: Dithered Fast Quantized JL}
\label{sec:d-fqjl}

A natural way to handle the dependence caused by reusing the same Gaussian
vector $g$ is to view the estimator as a function of $g$ and apply Gaussian
concentration for Lipschitz functions. This approach is not directly
applicable, however, because the quantizer $x\mapsto\sgn(x)$ is
discontinuous.

We therefore smooth the effect of the sign function using dithering, a
standard technique in which an independent random perturbation is added to a
signal before quantization. Specifically, D-FQJL uses non-subtractive
Gaussian dithering: the noise is added to the key sketch before quantization
and is not subtracted afterward. Although the output remains binary,
averaging over the dither replaces the discontinuous sign
response with a smooth function. Indeed, for $\xi\sim\mathcal{N}(0,1)$,
\begin{align}
\label{eq:d-fqjl-smoothed-sign}
s_\sigma(z) := \E_{\xi}\!\left[
\sgn\!\left(z+\sigma\xi\right)
\right]
&=
2\Phi\!\left(\frac{z}{\sigma}\right)-1,
\end{align}
where $\Phi$ is the standard Gaussian cumulative distribution function, and
if $\phi$ denotes its density, then
\begin{align*}
s_\sigma'(z)
&=
\frac{2}{\sigma}\phi\!\left(\frac{z}{\sigma}\right)
\end{align*}
Thus $s_\sigma$ is $\sqrt{2/\pi}/\sigma$-Lipschitz. This smoothing makes the
dependence on the shared Gaussian vector $g$ more amenable to Gaussian
concentration, subject to the localization described below.

\begin{definition}[D-FQJL estimator]
\label{def:d-fqjl-estimator}
Given a nonzero key $k\in\R^d$, a query $q\in\R^d$, and a noise level
$\sigma>0$, draw a dither $\xi\sim\mathcal{N}(0,I_m)$ independently of $T$
and define
\begin{align*}
\mathcal{Q}_\sigma(k)
&:= \sgn\!\bigl(Tk+\sigma\norm{k}_2\,\xi\bigr)
\in\{-1,+1\}^m.
\end{align*}
The D-FQJL estimator of $\ip{q}{k}$ is
\begin{align*}
\hat e_\sigma(q,k)
&:= \frac{\sqrt{\pi/2}\,\sqrt{1+\sigma^2}}{m}\,
\norm{k}_2\,\ip{Tq}{\mathcal{Q}_\sigma(k)}.
\end{align*}
\end{definition}


\begin{lemma}[Unbiasedness of D-FQJL]
\label{lem:d-fqjl-unbiasedness}
For every $q\in\R^d$, every nonzero $k\in\R^d$, and every $\sigma>0$,
\begin{align*}
\E_{g,\xi}\!\left[\hat e_\sigma(q,k)\mid I,D_\pm\right]
&= \ip{q}{k}.
\end{align*}
\end{lemma}
\noindent The proof is provided in the supplementary material.
 
We next explain how the smoothing identity enters the concentration analysis.
Assume for now that $\norm{q}_2=\norm{k}_2=1$. For each $i\in I$, we write
\begin{align}
\label{eq:dithered-sign-decomposition}
\sgn\!\bigl((Tk)_i+\sigma\xi_i\bigr)
&=
s_\sigma\bigl((Tk)_i\bigr)
\nonumber\\
&\quad+
\Bigl(
\sgn\!\bigl((Tk)_i+\sigma\xi_i\bigr)
-
s_\sigma\bigl((Tk)_i\bigr)
\Bigr).
\end{align}
This decomposition separates two qualitatively different contributions.
Conditioned on $T$, the term in parentheses is centered with respect to
$\xi_i$, bounded in absolute value by $2$, and independent across the dither
coordinates. Its contribution can therefore be bounded using concentration
for weighted sums of independent bounded random variables, together with
appropriate control of the weights $(Tq)_i$.

The smoothed term $s_\sigma((Tk)_i)$ depends on the shared Gaussian vector
$g$. Unlike the original sign quantizer, however, $s_\sigma$ is
$\sqrt{2/\pi}/\sigma$-Lipschitz. Its contribution to the estimator is smoother
as a function of $g$, but it also contains the Gaussian factor $(Tq)_i$.
Consequently, a direct application of Gaussian concentration is not
immediate, and the exact argument is more delicate. The supplementary
material provides the formal treatment of the required global and local
Lipschitz properties, as well as the corresponding weight control.

Thus, dithering separates the estimation error into a bounded fluctuation
generated by the independent dither variables and a smooth functional of the
shared Gaussian vector. Bounding these contributions separately and combining
the resulting estimates yields the following guarantee.


\begin{theorem}[D-FQJL distortion, informal]
\label{thm:d-fqjl-informal}
Fix $q\in\R^d$ and a nonzero $k\in\R^d$. Let
$\varepsilon,\delta,\sigma>0$. If
\begin{align*}
m
&\gtrsim
\frac{(1+\sigma^2)(1+1/\sigma)^2}{\varepsilon^2}
\log\frac{d}{\delta}\log\frac{1}{\delta},
\end{align*}
then
\begin{align*}
\Prob\!\left[
  \abs{\hat e_\sigma(q,k)-\ip{q}{k}}
  > \varepsilon\norm{q}_2\norm{k}_2
\right]
&\le \delta,
\end{align*}
where the probability is over the randomness of $I$, $D_\pm$, $g$, and
$\xi$.
\end{theorem}
\noindent The formal claim and proof are provided in the supplementary material.


The parameter $\sigma$ controls the tradeoff introduced by smoothing: larger
$\sigma$ makes the quantizer smoother but weakens the correlation between
$Tk$ and its stored signs, while smaller $\sigma$ preserves more signal but
increases the Lipschitz constant of the smoothed sign. The sample-complexity
bound captures this tradeoff through the factor
$(1+\sigma^2)(1+1/\sigma)^2$, which is minimized at $\sigma=1$.

\subsection{P-FQJL: Permutation-Based Fast Quantized JL}
\label{sec:p-fqjl}

Our second approach retains the original sign quantizer and instead
regularizes the dependence among the sampled rows. Unlike D-FQJL, it does not
perturb the key sketch before quantization. Instead, it inserts a uniformly
random permutation between the two Hadamard transforms and uses
\begin{align}
\label{eq:p-fqjl-transform}
T_{\Pi} &:= S_I\sqrt{d}\,H D_g \Pi H D_\pm,
\end{align}
where $\Pi$ is sampled uniformly from the $d\times d$ permutation matrices,
independent of $I$, $D_g$, and $D_\pm$. After $H D_\pm$ spreads the input,
$\Pi$ randomizes how its coordinates are paired with the shared Gaussian
entries in $D_g$.

\begin{definition}[P-FQJL estimator]
\label{def:p-fqjl-estimator}
Given a nonzero key $k\in\R^d$ and a query $q\in\R^d$, define
\begin{align*}
\mathcal{Q}_\Pi(k) &:= \sgn(T_\Pi k)\in\{-1,+1\}^m.
\end{align*}
The P-FQJL estimator of $\ip{q}{k}$ is
\begin{align*}
\hat e_\Pi(q,k)
&:= \frac{\sqrt{\pi/2}}{m}\,\norm{k}_2\,
\ip{T_\Pi q}{\mathcal{Q}_\Pi(k)}.
\end{align*}
\end{definition}

As in D-FQJL, the key is stored as an $m$-bit code together with its norm,
whereas the query is transformed on demand into the real-valued sketch
$T_\Pi q\in\R^m$.

\begin{lemma}[Unbiasedness of P-FQJL]
\label{lem:p-fqjl-unbiasedness}
For every $q\in\R^d$ and every nonzero $k\in\R^d$,
\begin{align*}
\E_g\!\left[\hat e_\Pi(q,k)\mid I,D_\pm,\Pi\right]
&= \ip{q}{k}.
\end{align*}
\end{lemma}
\noindent This follows directly from the Gaussian sign identity; the details
are provided in the supplementary material.

The permutation is not needed for unbiasedness; its role is to obtain
concentration. Conditioned on $I$, $D_\pm$, and $\Pi$, the pair
$(T_\Pi q,T_\Pi k)$ is jointly Gaussian over $g$. Define its covariance and
cross-covariance blocks by
\begin{align*}
\Sigma_{qq}
&:= \E_g\!\left[(T_\Pi q)(T_\Pi q)^\top\mid I,D_\pm,\Pi\right], \\
\Sigma_{kk}
&:= \E_g\!\left[(T_\Pi k)(T_\Pi k)^\top\mid I,D_\pm,\Pi\right], \\
\Sigma_{qk}
&:= \E_g\!\left[(T_\Pi q)(T_\Pi k)^\top\mid I,D_\pm,\Pi\right].
\end{align*}
The ideal blocks are $\norm{q}_2^2I_m$, $\norm{k}_2^2I_m$, and
$\ip{q}{k}I_m$, respectively. In this ideal case, the measurements are
uncorrelated across distinct sampled rows, while each pair
$((T_\Pi q)_i,(T_\Pi k)_i)$ has variances $\norm{q}_2^2$ and
$\norm{k}_2^2$ and covariance $\ip{q}{k}$. Since the measurements are jointly
Gaussian, the diagonal covariance structure would make the row pairs
independent and identically distributed. The random permutation approximately
restores this structure: after the first Hadamard transform spreads the input
energy, the permutation randomizes its alignment with the shared Gaussian
diagonal, making the off-diagonal correlations small and the three covariance
blocks close to their ideal values in operator norm.

\begin{lemma}[Covariance regularization, informal]
\label{clm:p-fqjl-covariance}
Let $q\in\R^d$, let $k\in\R^d$ be nonzero, and let
$\varepsilon,\delta>0$. If
\begin{align*}
m
&\lesssim
\frac{\varepsilon^2d}
{\log(d/\delta)\log(m/\delta)},
\end{align*}
then, with probability at least $1-\delta$ over $I$, $D_\pm$, and $\Pi$,
\begin{align*}
\norm{\Sigma_{qq}-\norm{q}_2^2I_m}_{\mathrm{op}}
&\le \varepsilon\norm{q}_2^2, \\
\norm{\Sigma_{kk}-\norm{k}_2^2I_m}_{\mathrm{op}}
&\le \varepsilon\norm{k}_2^2, \\
\norm{\Sigma_{qk}-\ip{q}{k}I_m}_{\mathrm{op}}
&\le \varepsilon\norm{q}_2\norm{k}_2.
\end{align*}
\end{lemma}
\noindent The formal claim and proof are provided in the supplementary material.

Once these blocks are close to their ideal values, the remaining argument no
longer depends on the fast-transform structure. Gaussian regression separates
the component of $T_\Pi q$ correlated with $T_\Pi k$ from an independent
Gaussian residual. Concentration of these two terms controls
$m^{-1}\ip{T_\Pi q}{\sgn(T_\Pi k)}$ and yields the following analogue of
Theorem~\ref{thm:d-fqjl-informal}.



\begin{theorem}[P-FQJL distortion, informal]
\label{thm:p-fqjl-informal}
Fix $q\in\R^d$ and a nonzero $k\in\R^d$. Let
$\varepsilon,\delta>0$. If
\begin{align*}
\frac{1}{\varepsilon^2}\log\frac{d}{\delta}
\lesssim m\lesssim
\frac{\varepsilon^2 d}{\log(d/\delta)\log(d\varepsilon^2/\delta)},
\end{align*}
then
\begin{align*}
\Prob\!\left[
  \abs{\hat e_\Pi(q,k)-\ip{q}{k}}
  >\varepsilon\norm{q}_2\norm{k}_2
\right]
&\le \delta,
\end{align*}
where the probability is over the randomness of $I$, $D_\pm$, $\Pi$, and
$g$. In particular, we can choose $m = \Theta(\frac{1}{\varepsilon^2}\log\frac{d}{\delta})$
and have the constraint.
\begin{align*}
  m \leq \sqrt{d/\log(d/\delta)}.
\end{align*}
\end{theorem}
\noindent The formal claim and proof are provided in the supplementary material.

Thus P-FQJL avoids the smoothing parameter required by D-FQJL, at the cost of
an upper bound on $m$ imposed by the permutation-based covariance argument.

\subsection{Space Complexity}
The memory cost has two distinct components: the shared sketching operator and
the vector-dependent representations. QJL uses a dense Gaussian
matrix in $\R^{m\times d}$ and therefore requires $md$ FPNs
to represent the operator.

In both D-FQJL and P-FQJL, the Hadamard matrices are implicit. The shared
D-FQJL operator stores only the Gaussian diagonal $g\in\R^d$, the random
signs defining $D_\pm$, and the $m$ sampled indices defining $S_I$. The
dither is generated while quantizing a key and need not be stored. P-FQJL
stores the same quantities together with $d$ integer indices specifying the
permutation $\Pi$. Thus, both operators require $O(d\!+\!m)$ storage.

For either construction, representing a key--query pair requires an $m$-bit
code and one FPN for the key norm, together with an $m$-dimensional
floating-point sketch of the query. Thus, the vector-dependent representation
uses $m$ bits and $m+1$ FPNs. As a consequence, a cache of $n$ keys requires
$nm$ bits and $n$ FPNs, in addition to the single shared structured operator;
the query sketch is computed on demand and need not be cached.

\section{The main application: KV-Cache compression}
\label{sec:fqjl-cache}
We now apply D-FQJL and P-FQJL to key-cache compression. At decoding step
$n$, let $\operatorname{score}(n)$ be the exact attention-score vector defined
in Equation~\eqref{eq:attention-scores}, and write
$\operatorname{score}_j(n)$ for its $j$-th entry.
Both algorithms store each key using an $m$-bit code and one scalar norm. The
sketching operator is sampled once and shared by all keys and queries. D-FQJL
additionally samples a fresh dither for each key, whereas P-FQJL samples one
permutation as part of the shared transform. The value cache is unchanged.

\begin{algorithm}[tb]
\caption{D-FQJL key-cache quantization}
\label{alg:d-fqjl-cache}
\begin{algorithmic}[1]
\REQUIRE Key stream $k_1,k_2,\ldots,k_n\in\R^d$, sketch size $m$, dither level
$\sigma>0$
\STATE Sample $I$, $D_\pm$, and $g$, and set
$T:=S_I\sqrt{d}\,H D_g H D_\pm$
\FOR{each arriving key $k_j$}
  \STATE Sample $\xi_j\sim\mathcal{N}(0,I_m)$ independently
  \STATE $(\mathcal{Q}_\sigma(k_j),\nu_j)
  \leftarrow
  (\sgn(Tk_j+\sigma\norm{k_j}_2\xi_j),\norm{k_j}_2)$
  \STATE Store $(\mathcal{Q}_\sigma(k_j),\nu_j)$
\ENDFOR
\STATE \textbf{Score estimation for a query $q$:}
\FOR{$j=1,\ldots,n$}
  \STATE $\widehat z_{\sigma,j}\leftarrow
  \dfrac{\sqrt{\pi/2}\sqrt{1+\sigma^2}}{m\sqrt{d}}
  \nu_j\ip{Tq}{\mathcal{Q}_\sigma(k_j)}$
\ENDFOR
\STATE \textbf{return}
$\widehat{\operatorname{score}}_\sigma(n)$
\end{algorithmic}
\end{algorithm}

The following results can be derived directly from the framework used by QJL.

\begin{theorem}[D-FQJL key-cache distortion, informal]
\label{thm:d-fqjl-cache}
Fix keys $k_1,\ldots,k_n\in\R^d$ and a query $q\in\R^d$ such that
$\max_{j\in[n]}\norm{k_j}_2\le r$ and $\norm{q}_2\le r$. Let
$\varepsilon,\delta,\sigma>0$. If
\begin{align*}
m
&\gtrsim
\frac{(1+\sigma^2)(1+1/\sigma)^2r^4}{\varepsilon^2d}
\log\frac{dn}{\delta}\log\frac{n}{\delta},
\end{align*}
then Algorithm~\ref{alg:d-fqjl-cache} satisfies, with probability at least
$1-\delta$, simultaneously for every $j\in[n]$,
\begin{align*}
\abs{\widehat{\operatorname{score}}_{\sigma,j}(n)
-\operatorname{score}_j(n)}
\lesssim\varepsilon\operatorname{score}_j(n).
\end{align*}
\end{theorem}
\noindent The formal claim and proof are provided in the supplementary material.

\begin{algorithm}[tb]
\caption{P-FQJL key-cache quantization}
\label{alg:p-fqjl-cache}
\begin{algorithmic}[1]
\REQUIRE Key stream $k_1,k_2,\ldots,k_n\in\R^d$, sketch size $m$
\STATE Sample $I$, $D_\pm$, $g$, and $\Pi$, and set
$T_\Pi:=S_I\sqrt{d}\,H D_g\Pi H D_\pm$
\FOR{each arriving key $k_j$}
  \STATE $(\mathcal{Q}_\Pi(k_j),\nu_j)
  \leftarrow(\sgn(T_\Pi k_j),\norm{k_j}_2)$
  \STATE Store $(\mathcal{Q}_\Pi(k_j),\nu_j)$
\ENDFOR
\STATE \textbf{Score estimation for a query $q$:}
\FOR{$j=1,\ldots,n$}
  \STATE $\widehat z_{\Pi,j}\leftarrow
  \dfrac{\sqrt{\pi/2}}{m\sqrt{d}}\nu_j
  \ip{T_\Pi q}{\mathcal{Q}_\Pi(k_j)}$
\ENDFOR
\STATE \textbf{return}
$\widehat{\operatorname{score}}_\Pi(n)$
\end{algorithmic}
\end{algorithm}

\begin{theorem}[P-FQJL key-cache distortion, informal]
\label{thm:p-fqjl-cache}
Fix keys $k_1,\ldots,k_n\in\R^d$ and a query $q\in\R^d$ such that
$\max_{j\in[n]}\norm{k_j}_2\le r$ and $\norm{q}_2\le r$. Let
$\varepsilon,\delta>0$. If
\begin{align*}
\frac{r^4}{\varepsilon^2d}\log\frac{n}{\delta}
\lesssim m\lesssim
\frac{\varepsilon^2d^2}
{r^4\log(dn/\delta)\log(mn/\delta)},
\end{align*}
then Algorithm~\ref{alg:p-fqjl-cache} satisfies, with probability at least
$1-\delta$, simultaneously for every $j\in[n]$,
\begin{align*}
\abs{\widehat{\operatorname{score}}_{\Pi,j}(n)
-\operatorname{score}_j(n)}
&\lesssim \varepsilon\operatorname{score}_j(n).
\end{align*}
\end{theorem}
\noindent The formal claim and proof are provided in the supplementary material.



Encoding all $n$ keys requires $O(nd\log d)$ time, corresponding to $n$
applications of the sketch $T$. Given the cached keys, computing the score of a single query
requires $O(d\log d+nm)$ time: $O(d\log d)$ to apply $T$ to the query and
$O(nm)$ to compute $n$ inner products in $\R^m$.

To quantize the value cache, we can apply standard quantization methods, such as those used in QJL, namely, the approaches proposed in \citep{liu2024kivi,hooper2024kvquant}.

\section{Connection to TurboQuant}
\label{sec:turboquant}

TurboQuant and the earlier DRIVE/EDEN line of work use random rotations to
spread vector energy before quantization
\citep{zandieh2026turboquant,vargaftik2022eden,ben2026note}. Applying a dense
rotation requires \(O(d^2)\) time per vector, so practical implementations
instead use an RHT, which can be applied in \(O(d\log d)\) time.

TurboQuant constructs a \(b\)-bit representation in two stages. The first
stage uses \((b-1)\)-bit quantization to obtain an approximation
\(\widetilde{x}\) of the input \(x\). It then forms the residual
\(r=x-\widetilde{x}\) and encodes it using 1-bit QJL, producing an \(m\)-bit
residual sketch. This residual correction is designed to improve the mean
squared error while approximately preserving inner products.

The theoretical QJL analysis uses a projection with independent Gaussian
rows and relies on Lemma~\ref{lem:gaussian-sign-identity}. In its experiments,
however, TurboQuant replaces this Gaussian projection with an RHT for
computational efficiency. Although this substitution performs well
empirically, it does not automatically inherit the unbiasedness
guarantee: the rows of \(H D_\pm\) are structured Rademacher vectors, and
orthogonality alone does not imply the Gaussian sign identity. This
theoretical distinction is also noted in \citep{ben2026note}.

Our constructions address this theoretical gap using Gaussian--Hadamard
transforms that retain 1-bit quantization while preserving the exact
sign--linear expectations required for unbiasedness. The resulting residual
sketch still contains \(m\) bits in total, but the transform can be applied in
\(O(d\log d)\) time. Our constructions therefore provide provably accurate
fast replacements for TurboQuant's dense residual QJL stage.

\section{Conclusion}

In this paper we provided fast quantized Johnson-Lindenstrauss transforms with corresponding asymmetric estimators. Our work provides running time and embedding dimensions analogous to ones in fast Johnson-Lindenstrauss transforms, proving that the power of dimensionality reduction persists under scalar quantization. 

\bibliographystyle{plainnat}
\bibliography{main}

\appendix
This document provides detailed proofs of the claims stated in the main
text. We begin with the probabilistic tools used throughout, then analyze
the fast quantized JL transforms, and conclude with the corresponding
unquantized Fast JL transform.

\section{Toolkit}

Throughout, $d$ is a power of two and $H\in\R^{d\times d}$ is the
normalized Walsh--Hadamard matrix, so $H$ is orthogonal and
$H_{ij}\in\{\pm1/\sqrt d\}$. Let
$D_\pm=\operatorname{diag}(\zeta)$ for an i.i.d. Rademacher vector
$\zeta\in\{\pm1\}^d$, and let $D_g=\operatorname{diag}(g)$ for
$g\sim\mathcal N(0,I_d)$. For a uniformly random subset
$I\subset[d]$ of size $m$, let $S_I$ denote the corresponding coordinate
selection matrix. We use the transforms
\begin{align}
  T&:=S_I\sqrt d\,H D_gH D_\pm,
  \label{eq:app-d-transform}\\
  T_\Pi&:=S_I\sqrt d\,H D_g\Pi H D_\pm,
  \label{eq:app-p-transform}
\end{align}
where $\Pi$ is an independent uniform permutation matrix.

For random vectors $U,V$ with finite second moments, we write
\begin{align*}
  \operatorname{Cov}(U,V)
  &:=\E\!\left[(U-\E U)(V-\E V)^{\mathsf T}\right],\\
  \operatorname{Cov}(U)
  &:=\operatorname{Cov}(U,U).
\end{align*}

We collect the probabilistic tools used in our analysis.

\begin{theorem}[Hoeffding's inequality \citep{Vershynin_2026}]
\label{thm:app-hoeffding}
Let $X_1,\ldots,X_n$ be independent, mean-zero random variables satisfying
$a_i\le X_i\le b_i$ almost surely. Then, for every $t>0$,
\begin{align*}
  \Prob\!\left[\abs{\sum_{i=1}^nX_i}\ge t\right]
  \le 2\exp\!\left(-\frac{2t^2}{\sum_{i=1}^n(b_i-a_i)^2}\right).
\end{align*}
\end{theorem}

\begin{theorem}[Gaussian concentration \citep{Vershynin_2026}]
\label{thm:app-gauss-conc}
Let $Z\sim\mathcal N(0,I_n)$ and let $f:\R^n\to\R$ be $L$-Lipschitz.
Then, for every $t>0$,
\begin{align*}
  \Prob\!\left[\abs{f(Z)-\E f(Z)}\ge t\right]
  \le 2\exp\!\left(-\frac{t^2}{2L^2}\right).
\end{align*}
\end{theorem}

We also use the following matrix Bernstein inequality for random
permutations.

\begin{theorem}[Matrix Bernstein inequality for random permutations \citep{mackey2014matrix}]
\label{thm:app-permutation-bernstein}
Let $\{A_{j\ell}\}_{j,\ell=1}^d$ be fixed self-adjoint $m\times m$ matrices
such that every row sum and every column sum is zero. Let $\pi$ be a uniform
random permutation of $[d]$, and define
\begin{align*}
  R:=\max_{j,\ell}\norm{A_{j\ell}}_{\mathrm{op}},
  \qquad
  v:=\frac1d\norm{\sum_{j,\ell=1}^d A_{j\ell}^2}_{\mathrm{op}}.
\end{align*}
Then, for every $t>0$,
\begin{align*}
  \Prob\!\left[\norm{\sum_{j=1}^d A_{j,\pi(j)}}_{\mathrm{op}}\ge t\right]
  \le 2m\exp\!\left(-\frac{t^2}{12v+4\sqrt2Rt}\right).
\end{align*}
\end{theorem}

Finally, we state the Gaussian quadratic-form concentration inequality
used in the analysis of the unquantized fast JL transform.

\begin{theorem}[Gaussian quadratic-form concentration \citep{Vershynin_2026}]
\label{thm:app-gaussian-quadratic}
Let $Z\sim\mathcal N(0,I_m)$ and let $B\in\R^{m\times m}$ be positive
semidefinite. Then, for every $s\ge0$,
\begin{align*}
  &\Prob\!\left[
    \abs{Z^{\mathsf T}BZ-\operatorname{tr}(B)}
    \ge 2\norm{B}_F\sqrt{s}+2\norm{B}_{\mathrm{op}}s
  \right] \\
  &\qquad\le 2e^{-s}.
\end{align*}
\end{theorem}

\section{Detailed Proofs of D-FQJL and P-FQJL}

Firstly, we prove the following elementary Gaussian identity that will be useful for both
estimators.

\begin{lemma}[Gaussian sign identity]
\label{lem:app-gaussian-sign}
Let $g\sim\mathcal{N}(0,I_d)$, and let $a,b\in\R^d$ with $b\neq0$. Then
\begin{align}
\label{eq:app-gaussian-sign-identity}
\E\!\left[
  \sqrt{\frac{\pi}{2}}\,\norm{b}_2\,
  \ip{a}{g}\,\sgn\!\bigl(\ip{b}{g}\bigr)
\right]
&= \ip{a}{b}.
\end{align}
\end{lemma}

\begin{proof}
Set $X:=\ip{a}{g} \sim \mathcal N(0,\norm{a}_2^2)$ and $Y:=\ip{b}{g} \sim \mathcal N(0,\norm{b}_2^2)$. Since $b\neq0$, decompose $a$
into components parallel and orthogonal to $b$. Then,
$X=\frac{\ip{a}{b}}{\norm{b}_2^2} Y+Z$, where $Z:=\ip{a_\perp}{g}$. The orthogonality of
$a_\perp$ and $b$ implies that the Gaussian projections $Z$ and $Y$ are
independent. Consequently,
\begin{align*}
  \E[X\sgn(Y)]
  &=\frac{\ip{a}{b}}{\norm{b}_2^2}\E[Y\sgn(Y)]+\E[Z]\E[\sgn(Y)] \\
  &=\frac{\ip{a}{b}}{\norm{b}_2^2}\E\abs{Y} \\
  &=\sqrt{\frac{2}{\pi}}\frac{\ip{a}{b}}{\norm{b}_2},
\end{align*}
because
$\E\abs{Y}=\sqrt{2/\pi}\,\norm{b}_2$. Multiplying by
$\sqrt{\pi/2}\,\norm{b}_2$ proves the claim.
\end{proof}

For completeness, we also record the obstruction for the first, unspread
Gaussian--Hadamard prototype from the main paper. Define
\begin{align*}
  T_0&:=S_I\sqrt d\,H D_g,\\
  \hat e_0(q,k)
  &:=\frac{\sqrt{\pi/2}}m\norm{k}_2
  \ip{T_0q}{\sgn(T_0k)}.
\end{align*}

\begin{proposition}[Variance of $\hat e_0$ need not vanish]
\label{prop:app-t0-variance}
There exist $q,k\in\R^d$ for which
$\operatorname{Var}(\hat e_0(q,k))$ is bounded below by an absolute
constant, independently of $m$.
\end{proposition}

\begin{proof}
Take $q=k=e_1$. Since the first column of $\sqrt d\,H$ has entries in
$\{\pm1\}$, every selected coordinate of $T_0e_1$ equals
$s_i g_1$ for some $s_i\in\{\pm1\}$. Hence
\begin{align*}
  \hat e_0(e_1,e_1)
  &=\frac{\sqrt{\pi/2}}m
    \sum_{i\in I}s_ig_1\sgn(s_ig_1)
    =\sqrt{\frac\pi2}\abs{g_1}.
\end{align*}
Using $\E\abs{g_1}=\sqrt{2/\pi}$ and $\E g_1^2=1$ gives
\begin{align*}
  \operatorname{Var}(\hat e_0(e_1,e_1))
  &=\frac\pi2\operatorname{Var}(\abs{g_1})
    =\frac\pi2\left(1-\frac2\pi\right)
    =\frac\pi2-1,
\end{align*}
which is independent of $m$.
\end{proof}

Since both constructions first apply an RHT, we prove the following lemma, which is standard~\citep{ailon2006approximate}.

\begin{lemma}[RHT spreading]
\label{lem:app-spreading}
Let $x\in\R^d$ be fixed and nonzero, and let $0<\delta<1/2$. Then
\begin{align*}
  \norm{HD_\pm x}_\infty
  \le \norm{x}_2\sqrt{\frac{2\log(16d/\delta)}{d}}
\end{align*}
with probability at least $1-\delta/8$. Consequently, the bound holds simultaneously for a fixed pair $q,k$ with probability at least
$1-\delta/4$.
\end{lemma}

\begin{proof}
For a fixed coordinate $i$,
$(HD_\pm x)_i=\sum_jH_{ij}\zeta_jx_j$ is a Rademacher sum. By Hoeffding's theorem~\ref{thm:app-hoeffding} we get
\begin{align*}
  \Prob\!\left[\abs{(HD_\pm x)_i}>t\right]
  \le 2\exp\!\left(-\frac{t^2}{2\norm{x}_2^2/d}\right),
\end{align*}
since $\sum_j(H_{ij}x_j)^2=\norm{x}_2^2/d$.

Take $t=\norm{x}_2\sqrt{2\log(16d/\delta)/d}$ and union bound over the $d$ coordinates. A second union bound in order to bound the maximum coordinate gives the desired statement.
\end{proof}

\subsection{D-FQJL: Dithered Fast Quantized JL}
\label{app:d-fqjl}
For fixed $q,k\in\R^d$, write
\begin{align*}
  \widetilde q:=HD_\pm q,
  \qquad
  \widetilde k:=HD_\pm k,
\end{align*}
i.e. the transformed version of $q,k$ after one RHT. 
Let $\xi=(\xi_i)_{i\in I}\sim\mathcal N(0,I_m)$ be an independent dither,
indexed by the sampled rows, and let $\sigma>0$ be the noise level. After the first application of RHT, $\widetilde q, \widetilde k$ will have small $\ell_\infty$ norm, as in previous Fast JL analyses. Since our final vectors are formed by linear combinations of the gaussian vectors $g$ (and hence jointly gaussian), we need to compute covariance matrices of Gaussian measurements. We define $M_q,M_k\in\R^{m\times d}$ matrices and the associated Gaussian vectors by
\begin{align*}
  M_q&:=S_I\sqrt d\,H\operatorname{diag}(\widetilde q),\\
  M_k&:=S_I\sqrt d\,H\operatorname{diag}(\widetilde k),\\
  A&:=M_qg,\\
  B&:=M_kg,\\
  B^\sigma&:=B+\sigma\norm{k}_2\xi.
\end{align*}

We index the coordinates of these vectors by the sampled rows, so
\begin{align*}
  A=(A_i)_{i\in I},
  \qquad
  B=(B_i)_{i\in I},
  \qquad
  B^\sigma=(B_i^\sigma)_{i\in I}.
\end{align*}
Since $S_I$ is a contraction and $H$ is orthogonal,
\begin{align*}
  \norm{M_q}_{\mathrm{op}}
  &\le\sqrt d\,\norm{\widetilde q}_\infty,
  &
  \norm{M_k}_{\mathrm{op}}
  &\le\sqrt d\,\norm{\widetilde k}_\infty.
\end{align*}
The D-FQJL estimator is
\begin{equation}
  \hat e_\sigma(q,k)
  :=\frac{\sqrt{\pi/2}\sqrt{1+\sigma^2}}{m}\norm{k}_2
  \ip{A}{\sgn(B^\sigma)}.
  \label{eq:app-d-estimator}
\end{equation}

\subsubsection{Unbiasedness of D-FQJL}

\begin{lemma}[D-FQJL Unbiasedness]
\label{lem:app-d-unbiasedness}
For every $q,k\in\R^d$, every $\sigma>0$, and every realization of
$(D_\pm,I)$,
\begin{align*}
  \E_{g,\xi}\!\left[\hat e_\sigma(q,k)\mid D_\pm,I\right]
  =\ip{q}{k}.
\end{align*}
\end{lemma}

\begin{proof}
The claim is immediate when $k=0$, so assume $k\ne0$ and condition on
$(D_\pm,I)$. For each $i\in I$, $(A_i,B_i^\sigma)$ is a centered jointly
Gaussian pair. Since $dH_{ij}^2=1$, $HD_\pm$ is orthogonal, and $\xi_i$ is
independent of $g$,
\begin{align*}
  A_i
  &\sim\mathcal N(0,\norm{q}_2^2),\\
  B_i^\sigma
  &\sim\mathcal N\!\left(0,(1+\sigma^2)\norm{k}_2^2\right),\\
  \E_{g,\xi_i}[A_iB_i^\sigma\mid D_\pm,I]
  &=\ip{\widetilde q}{\widetilde k}
   =\ip{q}{k}.
\end{align*}
The Gaussian regression decomposition therefore gives
\begin{align*}
  A_i
  =\frac{\ip{q}{k}}{(1+\sigma^2)\norm{k}_2^2}B_i^\sigma+Z_i,
\end{align*}
where $Z_i$ is a centered Gaussian variable independent of $B_i^\sigma$.
Consequently,
\begin{align*}
  \E_{g,\xi_i}[A_i\sgn(B_i^\sigma)\mid D_\pm,I]
  &=\frac{\ip{q}{k}}{(1+\sigma^2)\norm{k}_2^2}
    \E_{g,\xi_i}[\abs{B_i^\sigma}\mid D_\pm,I]\\
  &=\sqrt{\frac2\pi}
    \frac{\ip{q}{k}}{\sqrt{1+\sigma^2}\norm{k}_2},
\end{align*}
where we used
$\E\abs{B_i^\sigma}=\sqrt{2/\pi}\sqrt{1+\sigma^2}\norm{k}_2$.
Substituting this identity into \eqref{eq:app-d-estimator} and averaging
over the $m$ rows yields
\begin{align*}
  \E_{g,\xi}[\hat e_\sigma(q,k)\mid D_\pm,I]
  =\ip{q}{k}.
\end{align*}
\end{proof}

\subsubsection{Concentration of D-FQJL}

For the remaining D-FQJL concentration argument, by homogeneity assume
$\norm{q}_2=\norm{k}_2=1$. Define the smoothed sign
\begin{align*}
  s_\sigma(t)
  :=\E_{\xi\sim\mathcal N(0,1)}\sgn(t+\sigma\xi) 
  =\Prob\!\left[\xi\ge-\frac{t}{\sigma}\right]
    -\Prob\!\left[\xi<-\frac{t}{\sigma}\right] 
  =2\Phi\!\left(\frac{t}{\sigma}\right)-1,
\end{align*}
where we used the symmetry of the standard normal
distribution. It satisfies $\abs{s_\sigma(t)} \le 1$ and
$s_\sigma'(t)=\frac{2}{\sigma}\phi(t/\sigma)
=\frac{\sqrt{2/\pi}}{\sigma}
\exp\!\left(-\frac{t^2}{2\sigma^2}\right)$, with
$\abs{s_\sigma'(t)}\le\frac{\sqrt{2/\pi}}{\sigma}$.
Consequently,
\begin{align}
  \hat e_\sigma(q,k)-\ip{q}{k}
  =\sqrt{\frac\pi2}\sqrt{1+\sigma^2}\bigl(\mathcal D_\xi+\mathcal D_g\bigr).
  \label{eq:app-d-error-decomposition}
\end{align}

where
\begin{align*}
  F(g)&:=\frac1m\ip{M_qg}{s_\sigma(M_kg)},\\
  \mathcal D_\xi
  &:=\frac1m\ip{A}{\sgn(B+\sigma\xi)-s_\sigma(B)},\\
  \mathcal D_g
  &:=F(g)-\E_g[F(g)\mid D_\pm,I].
\end{align*}
Here $s_\sigma$ and $\sgn$ are applied coordinatewise to vectors.

The term $\mathcal D_\xi$ is a weighted sum of independent, centered, bounded dither fluctuations. Lemma~\ref{lem:app-row-energy} controls its weights, and Lemma~\ref{lem:app-dither} gives its concentration. The term $\mathcal D_g$ is the smoothed functional of the shared Gaussian vector;
Lemma~\ref{lem:app-smooth} controls it after the same row event. Combining these estimates proves the D-FQJL distortion theorem.

\begin{lemma}[Row-energy bound]
\label{lem:app-row-energy}
Let $0<\delta<1/2$. Suppose the spreading bound of
Lemma~\ref{lem:app-spreading} holds for $q$ and
\begin{align*}
  m\ge4\log(16d/\delta)\log(8/\delta).
\end{align*}
Then
\begin{align*}
  \norm{M_qg}_2^2\le4m
\end{align*}
with conditional probability at least $1-\delta/4$ over $g$, given
$(D_\pm,I)$.
\end{lemma}

\begin{proof}
Define $h:\R^d\to[0,\infty)$ by
\begin{align*}
  h(g):=\norm{M_qg}_2.
\end{align*}
By the reverse triangle inequality,
\begin{align*}
  \abs{h(g)-h(g')}
  &\le\norm{M_q(g-g')}_2\\
  &\le\norm{M_q}_{\mathrm{op}}\norm{g-g'}_2.
\end{align*}
Thus $h$ is $\norm{M_q}_{\mathrm{op}}$-Lipschitz. The operator-norm
bound above and Lemma~\ref{lem:app-spreading} give
\begin{align*}
  \norm{M_q}_{\mathrm{op}}
  \le\sqrt d\,\norm{\widetilde q}_\infty
  \le\sqrt{2\log(16d/\delta)}.
\end{align*}
Moreover, $\E_g h^2=m$ and
$\E_g h\le\sqrt m$. Therefore, Theorem~\ref{thm:app-gauss-conc} gives
\begin{align*}
  \Prob_g[h\ge\sqrt m+t\mid D_\pm,I]
  \le2\exp\!\left(-\frac{t^2}{4\log(16d/\delta)}\right).
\end{align*}
Taking
$t^2=4\log(16d/\delta)\log(8/\delta)$ makes the RHS $\delta/4$. The assumed lower bound on $m$ gives $t\le\sqrt m$, which completes the proof.
\end{proof}

\begin{lemma}[Dither concentration]
\label{lem:app-dither}
Let $0<\delta<1/2$. Condition on $(D_\pm,g,I)$ and suppose the bound in Lemma~\ref{lem:app-row-energy} holds. Then
\begin{align*}
  \abs{\mathcal D_\xi}
  \le4\sqrt{\frac{2\log(8/\delta)}{m}}
\end{align*}
with conditional probability at least $1-\delta/4$ over $\xi$.
\end{lemma}

\begin{proof}
For $i\in I$, set
\begin{align*}
  X_i
  :=A_i\bigl(\sgn(B_i+\sigma\xi_i)-s_\sigma(B_i)\bigr).
\end{align*}
Conditionally on $(D_\pm,g,I)$, the variables $X_i$ are independent and
centered, and
$-2\abs{A_i}\le X_i\le2\abs{A_i}$. The only randomness is introduced by $\xi$. Since
$\mathcal D_\xi=m^{-1}\sum_{i\in I}X_i$, Hoeffding's inequality gives
\begin{align*}
  \Prob_\xi\!\left[
    \abs{\sum_{i\in I}X_i}\ge t
    \;\middle|\;D_\pm,g,I
  \right]
  \le2\exp\!\left(-\frac{t^2}{32m}\right).
\end{align*}
Taking $t=4\sqrt{2m\log(8/\delta)}$ and dividing by $m$ proves the claim.
\end{proof}

The functional $F$ is not globally Lipschitz because the factor $M_qg$ is unbounded. We therefore localize to the row event $\{g:\norm{M_qg}_2\le2\sqrt m\}$. Replacing $M_qg$ by its projection onto the Euclidean ball of radius $2\sqrt m$ produces a globally Lipschitz functional that agrees with $F$ on this event. If the row event fails, the first alternative in the lemma holds. Otherwise, Gaussian concentration controls the projected functional, and the probability of the complement controls the difference between the two expectations.

\begin{lemma}[Concentration of the smoothed estimator]
\label{lem:app-smooth}
Let $0<\delta<1/2$. Suppose the spreading bound of
Lemma~\ref{lem:app-spreading} holds for both $q$ and $k$, and suppose
$m\ge4\log(16d/\delta)\log(8/\delta)$. Then, conditionally on
$(D_\pm,I)$, with probability at least $1-\delta/4$ over $g$, either
\begin{align*}
  \norm{M_qg}_2^2>4m
\end{align*}
or,
\begin{align*}
  &\abs{\mathcal D_g}
  \le
  2\left(1+\frac{2\sqrt{2/\pi}}{\sigma}\right)
  \sqrt{\frac{\log(16d/\delta)\log(8/\delta)}{m}}\\
  &\qquad\quad
  +\sqrt2\exp\!\left(-\frac{m}{8\log(16d/\delta)}\right).
\end{align*}
\end{lemma}

\begin{proof}
Condition on $(D_\pm,I)$ throughout and suppress this conditioning from the
notation. Define the event
\begin{align*}
  \mathcal G:=\{g:\norm{M_qg}_2\le2\sqrt m\}.
\end{align*}
Let $P$ be the Euclidean projection onto the ball of radius $2\sqrt m$ in
$\R^m$, and define
\begin{align*}
  \overline F(g):=\frac1m\ip{P(M_qg)}{s_\sigma(M_kg)}.
\end{align*}
The operator-norm estimate in the proof of
Lemma~\ref{lem:app-row-energy}, applied to both $q$ and $k$, gives
\begin{align*}
  \norm{M_q}_{\mathrm{op}},
  \norm{M_k}_{\mathrm{op}}
  &\le\sqrt{2\log(16d/\delta)}
\end{align*}
Since $P$ is $1$-Lipschitz, $\norm{P(M_qg)}_2\le2\sqrt m$,
$\norm{s_\sigma(M_kg)}_2\le\sqrt m$, and $s_\sigma$ is
$\sqrt{2/\pi}/\sigma$-Lipschitz,
\begin{align*}
  &\abs{\overline F(g)-\overline F(g')}\\
  &\le\frac{\norm{M_q}_{\mathrm{op}}
  +2\sqrt{2/\pi}\,\norm{M_k}_{\mathrm{op}}/\sigma}{\sqrt m}
  \norm{g-g'}_2\\
  &\le\left(1+\frac{2\sqrt{2/\pi}}{\sigma}\right)
  \sqrt{\frac{2\log(16d/\delta)}{m}}\norm{g-g'}_2.
\end{align*}
Thus Gaussian concentration implies, with probability at least
$1-\delta/4$,
\begin{align*}
  \abs{\overline F(g)-\E_g\overline F(g)}
  \le2\left(1+\frac{2\sqrt{2/\pi}}{\sigma}\right)
  \sqrt{\frac{\log(16d/\delta)\log(8/\delta)}{m}}.
\end{align*}

On $\mathcal G$, $F=\overline F$. Moreover,
\begin{align*}
  \abs{\E_gF-\E_g\overline F}
  &\le\E_g\abs{F-\overline F}\\
  &\le\frac1{\sqrt m}\E_g\!\left[
    \norm{M_qg}_2
    \mathbf1_{\{\norm{M_qg}_2>2\sqrt m\}}
  \right]\\
  &\le\left(\frac{\E_g\norm{M_qg}_2^2}{m}\right)^{1/2}
  \Prob_g\!\left[\norm{M_qg}_2>2\sqrt m\right]^{1/2}\\
  &\le\sqrt2\exp\!\left(-\frac{m}{8\log(16d/\delta)}\right).
\end{align*}
Here the second inequality uses $\abs{s_\sigma}\le1$, the third uses
Cauchy--Schwarz and $\E_g\norm{M_qg}_2^2=m$, and the last follows from
the tail estimate in Lemma~\ref{lem:app-row-energy}.
On $\mathcal G$, we therefore have
\begin{align*}
  \abs{\mathcal D_g}
  &=\abs{F(g)-\E_gF(g)}\\
  &\le\abs{\overline F(g)-\E_g\overline F(g)}
    +\abs{\E_g\overline F(g)-\E_gF(g)}.
\end{align*}
Combining the preceding two estimates proves the stated alternative.
\end{proof}

\begin{theorem}[D-FQJL distortion]
\label{thm:app-d-fqjl}
Fix $q\in\R^d$ and a nonzero $k\in\R^d$, and let
$0<\varepsilon<1$, $\sigma>0$, and $0<\delta<1/2$. Suppose $m\le d$ and
\begin{align}
  m&\ge
  \frac{144\pi(1+\sigma^2)}{\varepsilon^2}
  \left(1+\frac{2\sqrt{2/\pi}}{\sigma}\right)^2
  \log\frac{16d}{\delta}\log\frac8\delta,
  \label{eq:app-d-m}
\end{align}
then
\begin{align*}
  \Prob\!\left[
  \abs{\hat e_\sigma(q,k)-\ip{q}{k}}
  >\varepsilon\norm{q}_2\norm{k}_2\right]\le\delta.
\end{align*}
The probability is over $I,D_\pm,g$, and $\xi$.
\end{theorem}

\begin{proof}
The case $q=0$ is immediate. By homogeneity, assume
$\norm{q}_2=\norm{k}_2=1$.

The dimension assumption implies
\begin{align*}
  m\ge4\log\frac{16d}{\delta}\log\frac8\delta,
\end{align*}
so all preceding concentration lemmas apply. Let $\mathcal E_1$ be the event for $q$ and $k$ from Lemma~\ref{lem:app-spreading}, and let $\mathcal E_2$ be the event $\norm{M_qg}_2^2\le4m$. Let $\mathcal E_3$ denote the bound on
$\mathcal D_g$ in Lemma~\ref{lem:app-smooth}, and let $\mathcal E_4$
denote the bound on $\mathcal D_\xi$ in Lemma~\ref{lem:app-dither}.
Lemmas~\ref{lem:app-spreading}, \ref{lem:app-row-energy},
\ref{lem:app-smooth}, and \ref{lem:app-dither}, respectively, give
\begin{align*}
  \Prob(\mathcal E_1^c)&\le\frac\delta4,\\
  \Prob(\mathcal E_2^c\mid\mathcal E_1)&\le\frac\delta4,\\
  \Prob(\mathcal E_2\cap\mathcal E_3^c\mid\mathcal E_1)
  &\le\frac\delta4,\\
  \Prob(\mathcal E_4^c\mid\mathcal E_2)&\le\frac\delta4.
\end{align*}
A sequential union bound shows that
$\mathcal E_1\cap\mathcal E_2\cap\mathcal E_3\cap\mathcal E_4$ holds
with probability at least $1-\delta$.

On this event, \eqref{eq:app-d-error-decomposition},
Lemma~\ref{lem:app-dither}, and Lemma~\ref{lem:app-smooth} give
\begin{align*}
  &\abs{\hat e_\sigma(q,k)-\ip{q}{k}}\\
  &\le\sqrt{\frac\pi2}\sqrt{1+\sigma^2}\Biggl[
  4\sqrt{\frac{2\log(8/\delta)}{m}}\\
  &\quad+2\left(1+\frac{2\sqrt{2/\pi}}{\sigma}\right)
  \sqrt{\frac{\log(16d/\delta)\log(8/\delta)}{m}}\\
  &\quad+\sqrt2\exp\!\left(
  -\frac{m}{8\log(16d/\delta)}
  \right)\Biggr].
\end{align*}
The dimension bound \eqref{eq:app-d-m}, together with
$\log(16d/\delta)\ge1$ and
$1+2\sqrt{2/\pi}/\sigma\ge1$, implies
\begin{align*}
  4\sqrt{\frac\pi2}\sqrt{1+\sigma^2}
  \sqrt{\frac{2\log(8/\delta)}{m}}
  &\le\frac\varepsilon3,\\
  2\sqrt{\frac\pi2}\sqrt{1+\sigma^2}
  \left(1+\frac{2\sqrt{2/\pi}}{\sigma}\right)
  \sqrt{\frac{\log(16d/\delta)\log(8/\delta)}{m}}
  &\le\frac\varepsilon3.
\end{align*}
For the final term, $\log x\le x^2$ for $x\ge1$ and
$\log(8/\delta)>1/2$ imply from \eqref{eq:app-d-m} that
\begin{align*}
  \sqrt{\pi(1+\sigma^2)}
  \exp\!\left(-\frac{m}{8\log(16d/\delta)}\right)
  \le\frac\varepsilon3.
\end{align*}
Combining the three estimates gives
\begin{align*}
  \abs{\hat e_\sigma(q,k)-\ip{q}{k}}
  \le\varepsilon.
\end{align*}
Rescaling by $\norm{q}_2\norm{k}_2$ completes the proof.
\end{proof}

Since
$1+2\sqrt{2/\pi}/\sigma\le2(1+1/\sigma)$, the dimension assumption has
the form
\begin{align*}
  m
  \gtrsim
  \frac{(1+\sigma^2)(1+1/\sigma)^2}{\varepsilon^2}
  \log\frac{d}{\delta}\log\frac1\delta,
\end{align*}
which is the informal bound stated in the main paper.

\subsection{Permutation-Based Fast Quantized JL}
\label{app:p-fqjl}

We first prove a fixed-pair estimate with an explicit failure probability;
the key-cache guarantee then follows by a union bound. Fix $q,k\in\R^d$ and
write $X:=T_\Pi q$ and $Y:=T_\Pi k$.
The P-FQJL estimator is
\begin{equation}
  \hat e_\Pi(q,k)
  :=\frac{\sqrt{\pi/2}}m\norm{k}_2\ip{X}{\sgn(Y)}.
  \label{eq:app-p-estimator}
\end{equation}
Conditionally on $(I,D_\pm,\Pi)$, the only remaining randomness is the
Gaussian diagonal $g$, so $(X,Y)$ is jointly Gaussian. We use the conditional
covariance blocks
\begin{align*}
  \Sigma_{qq}&:=\E_g[XX^{\mathsf T}],\\
  \Sigma_{kk}&:=\E_g[YY^{\mathsf T}],\\
  \Sigma_{qk}&:=\E_g[XY^{\mathsf T}].
\end{align*}
The proof has three stages: the permutation regularizes these blocks,
Gaussian regression separates the part of $X$ correlated with $Y$, and
Gaussian concentration controls the resulting scalar errors.

\subsubsection{Unbiasedness of P-FQJL}

\begin{lemma}[Unbiasedness of P-FQJL]
\label{lem:app-p-unbiasedness}
For every realization of $(I,D_\pm,\Pi)$ and every $q,k\in\R^d$,
\begin{equation*}
  \E_g\hat e_\Pi(q,k)=\ip qk.
\end{equation*}
\end{lemma}

\begin{proof}
The case $k=0$ is immediate. Suppose $k\ne0$ and condition on
$(I,D_\pm,\Pi)$. Let $r_i^{\mathsf T}$ be a selected row of
$S_I\sqrt d\,H$, and define
\begin{equation*}
  a_i:=r_i\odot(\Pi HD_\pm q),
  \qquad
  b_i:=r_i\odot(\Pi HD_\pm k).
\end{equation*}
Then $X_i=\ip{a_i}{g}$ and $Y_i=\ip{b_i}{g}$. Since the entries of $r_i$
belong to $\{\pm1\}$ and both $\Pi$ and $HD_\pm$ are orthogonal, we have
$\norm{b_i}_2=\norm{k}_2$ and $\ip{a_i}{b_i}=\ip qk$. Hence
Lemma~\ref{lem:app-gaussian-sign} gives
\begin{equation*}
  \E_g[X_i\sgn(Y_i)]
  =\sqrt{\frac2\pi}\frac{\ip qk}{\norm{k}_2}.
\end{equation*}
Summing over the selected rows and substituting into
\eqref{eq:app-p-estimator} proves the claim.
\end{proof}

\subsubsection{Concentration of P-FQJL}

\begin{lemma}[Covariance regularization by permutation]
\label{lem:app-p-covariance}
Let $q,k\in\R^d$, $0<\varepsilon\le1$, and $0<\delta<1/2$. If
\begin{equation*}
  m\le
  \frac{\varepsilon^2d}
  {(24+12\sqrt2)\log(16d/\delta)\log(8m/\delta)},
\end{equation*}
then, with probability at least $1-\delta$ over $(D_\pm,\Pi)$,
\begin{align*}
  \norm{\Sigma_{qq}-\norm q_2^2I_m}_{\mathrm{op}}
  &\le\varepsilon\norm q_2^2,\\
  \norm{\Sigma_{kk}-\norm k_2^2I_m}_{\mathrm{op}}
  &\le\varepsilon\norm k_2^2,\\
  \norm{\Sigma_{qk}-\ip qk I_m}_{\mathrm{op}}
  &\le\varepsilon\norm q_2\norm k_2.
\end{align*}
\end{lemma}

\begin{proof}
We prove the cross-covariance estimate; the two diagonal estimates follow by
applying the same argument to $(q,q)$ and $(k,k)$. The zero-vector cases are
immediate, so assume $q,k\ne0$ and condition on an arbitrary realization of
$I$. Set $L:=\log(16d/\delta)$,
$\widetilde q:=HD_\pm q$, and $\widetilde k:=HD_\pm k$. Let
$h_1,\ldots,h_d$ be the columns of $S_I\sqrt d\,H$. Then
\begin{equation*}
  \norm{h_j}_2^2=m,
  \qquad
  \sum_{j=1}^dh_jh_j^{\mathsf T}=dI_m.
\end{equation*}

\paragraph{Step 1: Center the permutation sum.}
With the convention $(\Pi x)_j=x_{\pi(j)}$, averaging over $g$ gives
\begin{equation*}
  \Sigma_{qk}
  =\sum_{j=1}^d
    \widetilde q_{\pi(j)}\widetilde k_{\pi(j)}h_jh_j^{\mathsf T}.
\end{equation*}
Define
$\alpha_\ell:=\widetilde q_\ell\widetilde k_\ell-\ip qk/d$ and
$C_j:=h_jh_j^{\mathsf T}-I_m$. Orthogonality gives
$\sum_\ell\alpha_\ell=0$ and $\sum_jC_j=0$. Therefore the array
$A_{j\ell}:=\alpha_\ell C_j$ has zero row and column sums, and
\begin{equation}
  \sum_{j=1}^dA_{j,\pi(j)}
  =\Sigma_{qk}-\ip qk I_m.
  \label{eq:app-p-permutation-sum}
\end{equation}

\paragraph{Step 2: Use the first RHT to bound the coefficients.}
Lemma~\ref{lem:app-spreading} implies, with probability at least
$1-\delta/4$ over $D_\pm$, that
\begin{equation*}
  \norm{\widetilde q}_\infty
  \le\norm q_2\sqrt{\frac{2L}{d}},
  \qquad
  \norm{\widetilde k}_\infty
  \le\norm k_2\sqrt{\frac{2L}{d}}.
\end{equation*}
On this event,
\begin{align*}
  \norm{\alpha}_\infty
  &\le\frac{3L}{d}\norm q_2\norm k_2,\\
  \norm{\alpha}_2^2
  &=\sum_{\ell=1}^d\widetilde q_\ell^2\widetilde k_\ell^2
    -\frac{\ip{q}{k}^2}{d}
  \le\frac{2L}{d}\norm q_2^2\norm k_2^2.
\end{align*}
The first estimate follows by bounding the product and centering terms
separately; the second uses
$\sum_\ell\widetilde q_\ell^2\widetilde k_\ell^2
 \le\norm{\widetilde q}_\infty^2\norm{\widetilde k}_2^2$.

\paragraph{Step 3: Apply permutation Bernstein.}
Since $(h_jh_j^{\mathsf T})^2=m h_jh_j^{\mathsf T}$,
\begin{equation*}
  \norm{C_j}_{\mathrm{op}}\le m,
  \qquad
  \sum_{j=1}^dC_j^2=d(m-1)I_m.
\end{equation*}
The Bernstein parameters therefore satisfy
\begin{align*}
  R_0&\le\frac{3mL}{d}\norm q_2\norm k_2,\\
  v_0&\le\frac{2mL}{d}\norm q_2^2\norm k_2^2.
\end{align*}
Taking $t=\varepsilon\norm q_2\norm k_2$ and using
$\varepsilon\le1$ gives
\begin{equation*}
  12v_0+4\sqrt2R_0t
  \le(24+12\sqrt2)\frac{mL}{d}
  \norm q_2^2\norm k_2^2.
\end{equation*}
Consequently,
\begin{equation*}
  \Prob_\Pi\!\left[
    \norm{\Sigma_{qk}-\ip qk I_m}_{\mathrm{op}}
    \ge\varepsilon\norm q_2\norm k_2
    \;\middle|\;D_\pm,I
  \right]
  \le2m\exp\!\left(
    -\frac{\varepsilon^2d}{(24+12\sqrt2)mL}
  \right)
  \le\frac\delta4.
\end{equation*}
The dimension assumption makes the exponent at least
$\log(8m/\delta)$. The same conditional estimate applies to both diagonal
covariances. A union bound over these three estimates and the common
spreading event proves the claim.
\end{proof}

\begin{theorem}[P-FQJL distortion]
\label{thm:app-p-fqjl}
Fix $q\in\R^d$ and a nonzero $k\in\R^d$, and let
$0<20\varepsilon<1$ and $0<\delta<1/2$. If
\begin{align*}
  m
  &\ge\frac{1}{5\varepsilon^2}\log\frac{d}{\delta},\\
  m
  &\le\frac{\varepsilon^2d}
  {(24+12\sqrt2)
   \log(32d/\delta)\log(6400d\varepsilon^2/\delta)},
\end{align*}
then
\begin{equation*}
  \Prob\!\left[
    \abs{\hat e_\Pi(q,k)-\ip{q}{k}}
    >20\varepsilon\norm{q}_2\norm{k}_2
  \right]\le\delta.
\end{equation*}
The probability is over $I,D_\pm,\Pi$, and $g$.
\end{theorem}

\begin{proof}
The assertion is immediate when $q=0$. By homogeneity, it suffices to
assume $\norm{q}_2=\norm{k}_2=1$. Condition on an arbitrary realization of
$I$. Set $\rho:=\ip qk$, $u:=q-\rho k$, and $W:=T_\Pi u$. Then
$u\perp k$, $\norm{u}_2^2=1-\rho^2\le1$, and
\begin{equation}
  X=\rho Y+W.
  \label{eq:app-p-orthogonal-decomposition}
\end{equation}

\paragraph{Step 1: Covariance event.}
Since $I\subset[d]$ has size $m$, we have $m\le d$. Together with the
lower bound on $m$, this gives
\begin{equation*}
  d\varepsilon^2
  \ge\frac15\log\frac d\delta
  \ge\frac15\log2.
\end{equation*}
Consequently, both logarithms in the upper bound on $m$ are greater than
one, so that bound implies $m\le d\varepsilon^2$, and hence
\begin{equation*}
  \log\frac{16m}{\delta}
  \le\log\frac{16d\varepsilon^2}{\delta}.
\end{equation*}
We may therefore apply Lemma~\ref{lem:app-p-covariance} to $(u,k)$ with
accuracy $\varepsilon$ and failure probability $\delta/2$. Since
$\ip{u}{k}=0$, with probability at least $1-\delta/2$ over
$(D_\pm,\Pi)$, the event $\mathcal E_{\mathrm{cov}}$ holds on which
\begin{align}
  \norm{\operatorname{Cov}(W)-\norm{u}_2^2I_m}_{\mathrm{op}}
  &\le\varepsilon\norm{u}_2^2,
  \label{eq:app-p-w-covariance}\\
  \norm{\Sigma_{kk}-I_m}_{\mathrm{op}}
  &\le\varepsilon,
  \label{eq:app-p-y-covariance}\\
  \norm{\operatorname{Cov}(W,Y)}_{\mathrm{op}}
  &\le\varepsilon\norm{u}_2.
  \label{eq:app-p-cross-covariance}
\end{align}
Moreover, the rowwise variance calculation from
Lemma~\ref{lem:app-p-unbiasedness} gives
$(\Sigma_{kk})_{ii}=1$ for every $i$.

\paragraph{Step 2: Gaussian regression.}
Condition on $\mathcal E_{\mathrm{cov}}$. By
\eqref{eq:app-p-y-covariance}, $\Sigma_{kk}$ is invertible and
\begin{equation*}
  \norm{\Sigma_{kk}^{-1}}_{\mathrm{op}}
  \le\frac{1}{1-\varepsilon}
  \le\frac{20}{19}.
\end{equation*}
Write $\Gamma:=\operatorname{Cov}(W,Y)$. The conditional-distribution
formula for a centered jointly Gaussian pair therefore yields
\begin{equation*}
  W\mid Y
  \sim\mathcal N\!\left(
    \Gamma\Sigma_{kk}^{-1}Y,
    \operatorname{Cov}(W)-\Gamma\Sigma_{kk}^{-1}\Gamma^{\mathsf T}
  \right).
\end{equation*}
Equivalently, we may write
\begin{equation}
  W=BY+R,
  \qquad
  B:=\Gamma\Sigma_{kk}^{-1},
  \label{eq:app-p-gaussian-regression}
\end{equation}
where $R$ is a centered Gaussian vector independent of $Y$, with
\begin{equation*}
  \operatorname{Cov}(R)
  =\operatorname{Cov}(W)
  -\Gamma\Sigma_{kk}^{-1}\Gamma^{\mathsf T}.
\end{equation*}
The covariance bounds imply
\begin{align}
  \norm{B}_{\mathrm{op}}
  &\le\frac{20\varepsilon}{19}\norm{u}_2
  \le\frac{20\varepsilon}{19},
  \label{eq:app-p-regression-bound}\\
  \operatorname{Cov}(R)
  &\preceq\operatorname{Cov}(W)
  \preceq(1+\varepsilon)
  \norm{u}_2^2I_m
  \preceq\frac{21}{20}I_m.
  \label{eq:app-p-residual-covariance}
\end{align}

\paragraph{Step 3: Decompose the scalar error.}
Substituting \eqref{eq:app-p-gaussian-regression} into
\eqref{eq:app-p-orthogonal-decomposition} gives
$X=\rho Y+BY+R$, and consequently
\begin{align}
  \frac1m\ip{X}{\sgn(Y)}-\sqrt{\frac2\pi}\rho
  &=
  \rho\left(
    \frac1m\sum_{i=1}^m\abs{Y_i}-\sqrt{\frac2\pi}
  \right)\nonumber\\
  &\quad+\frac1m\ip{BY}{\sgn(Y)}
  +\frac1m\ip{R}{\sgn(Y)}.
  \label{eq:app-p-error-decomposition}
\end{align}

\paragraph{Step 4: Concentrate the three terms.}
Write $Y=\Sigma_{kk}^{1/2}G$, where $G\sim\mathcal N(0,I_m)$, and define
\begin{equation*}
  f_1(G):=\frac1m\norm{\Sigma_{kk}^{1/2}G}_1,
  \qquad
  f_2(G):=\frac1{\sqrt m}\norm{\Sigma_{kk}^{1/2}G}_2.
\end{equation*}
Since
$(\Sigma_{kk})_{ii}=1$, each $Y_i$ is standard normal, and therefore
\begin{align*}
  \E f_1(G)&=\sqrt{\frac2\pi},
  &
  \E f_2(G)&\le1.
\end{align*}
Using \eqref{eq:app-p-y-covariance}, the inequality
$\norm{z}_1\le\sqrt m\norm{z}_2$, and the reverse triangle inequality, both
functions satisfy
\begin{equation*}
  \abs{f_j(G)-f_j(G')}
  \le\sqrt{\frac{21}{20m}}\norm{G-G'}_2,
  \qquad j\in\{1,2\}.
\end{equation*}
Set
\begin{equation*}
  a:=\sqrt{\frac{21\log(12/\delta)}{10m}}.
\end{equation*}
Gaussian concentration shows that each of the following inequalities fails
with probability at most $\delta/6$:
\begin{align}
  \abs{\frac1m\sum_{i=1}^m\abs{Y_i}-\sqrt{\frac2\pi}}
  &\le a,
  \label{eq:app-p-absolute-concentration}\\
  \frac{\norm{Y}_2}{\sqrt m}
  &\le1+a.
  \label{eq:app-p-norm-concentration}
\end{align}

Conditionally on $Y$, the independence in
\eqref{eq:app-p-gaussian-regression} and
\eqref{eq:app-p-residual-covariance} show that
$m^{-1}\ip{R}{\sgn(Y)}$ is centered Gaussian with conditional variance at
most
\begin{equation*}
  \frac1{m^2}\sgn(Y)^{\mathsf T}
  \operatorname{Cov}(R)\sgn(Y)
  \le\frac{21}{20m}.
\end{equation*}
Hence, except on a further event of probability at most $\delta/6$,
\begin{align}
  \abs{\frac1m\ip{R}{\sgn(Y)}}
  \le a.
  \label{eq:app-p-residual-concentration}
\end{align}
Finally, \eqref{eq:app-p-regression-bound} and
\eqref{eq:app-p-norm-concentration} give
\begin{equation*}
  \abs{\frac1m\ip{BY}{\sgn(Y)}}
  \le\norm{B}_{\mathrm{op}}\frac{\norm{Y}_2}{\sqrt m}
  \le\frac{20\varepsilon}{19}(1+a).
\end{equation*}

Combining these estimates in \eqref{eq:app-p-error-decomposition} and
using $\abs{\rho}\le1$ yields
\begin{equation*}
  \abs{\frac1m\ip{X}{\sgn(Y)}
  -\sqrt{\frac2\pi}\rho}
  \le
  2a+\frac{20\varepsilon}{19}(1+a).
\end{equation*}
Since $d\ge1$ and $\delta<1/2$,
\begin{equation*}
  \log\frac{12}{\delta}
  \le\frac{23}{5}\log\frac d\delta.
\end{equation*}
The lower bound on $m$ therefore gives
\begin{equation*}
  a\le7\varepsilon.
\end{equation*}
Using $\varepsilon<1/20$, the last error display is at most
\begin{equation*}
  \left[
    14+
    \frac{20}{19}\left(1+\frac7{20}\right)
  \right]\varepsilon
  <20\sqrt{\frac2\pi}\,\varepsilon.
\end{equation*}
In the normalized case,
$\hat e_\Pi(q,k)=\sqrt{\pi/2}\,m^{-1}\ip{X}{\sgn(Y)}$, so
\begin{equation*}
  \abs{\hat e_\Pi(q,k)-\ip{q}{k}}
  \le20\varepsilon.
\end{equation*}
The covariance event fails with probability at most $\delta/2$, and the
three conditional concentration events each fail with probability at most
$\delta/6$. Thus the total failure probability is at most $\delta$. Since
the argument holds for every realization of $I$, rescaling by
$\norm{q}_2\norm{k}_2$ completes the proof.
\end{proof}

\subsection{Key-Cache Distortion Guarantees}
\label{app:key-cache}

We now convert the inner-product estimates into attention-score guarantees
using a union bound followed directly by the softmax calculation.


\begin{theorem}[D-FQJL key-cache distortion]
\label{thm:app-d-fqjl-cache}
Fix $n\in\mathbb N$, $r>0$, keys $k_1,\ldots,k_n\in\R^d$, and a query
$q\in\R^d$ satisfying
$\max_j\norm{k_j}_2\le r$ and $\norm q_2\le r$. Let
$0<\varepsilon\le1$, $0<\delta<1/2$, and $\sigma>0$, and assume
$\varepsilon\sqrt d<4r^2$.
Suppose $m\le d$ and
\begin{align*}
  m
  \ge
  \frac{2304\pi(1+\sigma^2)r^4}{\varepsilon^2d}
  \left(1+\frac{2\sqrt{2/\pi}}{\sigma}\right)^2
  \log\frac{16dn}{\delta}\log\frac{8n}{\delta}.
\end{align*}
Then, with probability at least $1-\delta$, simultaneously for every
$j\in[n]$,
\begin{align*}
  \abs{\widehat{\operatorname{score}}_{\sigma,j}(n)
  -\operatorname{score}_j(n)}
  \le\varepsilon\operatorname{score}_j(n).
\end{align*}
\end{theorem}

\begin{proof}
The zero cases are immediate. Apply Theorem~\ref{thm:app-d-fqjl} to each
nonzero pair $(q,k_j)$ with distortion
$\varepsilon\sqrt d/(4r^2)$ and failure probability $\delta/n$. The
displayed dimension assumption is precisely the resulting condition in that
theorem. A union bound therefore gives, with probability at least
$1-\delta$, simultaneously for every $j$,
\begin{align*}
  \abs{\hat e_{\sigma,j}(q,k_j)-\ip q{k_j}}
  &\le\frac{\varepsilon\sqrt d}{4r^2}
  \norm q_2\norm{k_j}_2\\
  &\le\frac{\varepsilon\sqrt d}{4r^2}r^2
  =\frac{\varepsilon\sqrt d}{4}.
\end{align*}
Thus every logit has error at most $\varepsilon/4$. After applying softmax,
the numerator and normalizing denominator each change by a factor in
$[e^{-\varepsilon/4},e^{\varepsilon/4}]$. Hence each score changes by a
factor in $[e^{-\varepsilon/2},e^{\varepsilon/2}]$, which proves the claim
since $e^{\varepsilon/2}-1\le\varepsilon$ for
$0<\varepsilon\le1$.
\end{proof}

\begin{theorem}[P-FQJL key-cache distortion]
\label{thm:app-p-fqjl-cache}
Fix $n\in\mathbb N$, $r>0$, keys $k_1,\ldots,k_n\in\R^d$, and a query
$q\in\R^d$ satisfying
$\max_j\norm{k_j}_2\le r$ and $\norm q_2\le r$. Let
$0<20\varepsilon\le1$ and $0<\delta<1/2$, and assume
$20\varepsilon\sqrt d<4r^2$.
If
\begin{align*}
  m
  &\ge\frac{16r^4}{5\varepsilon^2d}\log\frac{dn}{\delta},\\
  m
  &\le
  \frac{\varepsilon^2d^2}
  {16(24+12\sqrt2)r^4
   \log(32dn/\delta)}\\[-2pt]
  &\qquad{}\times\frac{1}
   {\log\!\left(400\varepsilon^2d^2n/(r^4\delta)\right)},
\end{align*}
then, with probability at least $1-\delta$, simultaneously for every
$j\in[n]$,
\begin{align*}
  \abs{\widehat{\operatorname{score}}_{\Pi,j}(n)
  -\operatorname{score}_j(n)}
  \le20\varepsilon\operatorname{score}_j(n).
\end{align*}
\end{theorem}

\begin{proof}
The zero cases are immediate. Set
\begin{equation*}
  \eta:=\frac{\varepsilon\sqrt d}{4r^2}.
\end{equation*}
The assumption $20\varepsilon\sqrt d<4r^2$ ensures $20\eta<1$.
Apply Theorem~\ref{thm:app-p-fqjl} to each nonzero pair $(q,k_j)$ with
parameter $\eta$ and failure probability $\delta/n$. Substituting $\eta$
into that theorem gives exactly the two displayed dimension conditions. A
union bound therefore gives, with probability at least
$1-\delta$, simultaneously for every $j$,
\begin{align*}
  \abs{\hat e_\Pi(q,k_j)-\ip q{k_j}}
  &\le20\eta\norm q_2\norm{k_j}_2\\
  &\le\frac{20\varepsilon\sqrt d}{4r^2}r^2
  =5\varepsilon\sqrt d.
\end{align*}
After division by $\sqrt d$, every logit has error at most $5\varepsilon$.
Thus each exponential in the softmax numerator and denominator changes by a
factor in $[e^{-5\varepsilon},e^{5\varepsilon}]$, so their ratio changes by
a factor in $[e^{-10\varepsilon},e^{10\varepsilon}]$. The result follows
from $e^{10\varepsilon}-1\le20\varepsilon$ for
$0<20\varepsilon\le1$.
\end{proof}

\end{document}